\documentclass[nonatbib,final]{article}
\usepackage{nips_2016}
\usepackage[utf8]{inputenc} % allow utf-8 input
\usepackage[T1]{fontenc}    % use 8-bit T1 fonts
\usepackage{hyperref}       % hyperlinks
\usepackage{booktabs}       % professional-quality tables
\usepackage{nicefrac}       % compact symbols for 1/2, etc.
\usepackage{microtype}      % microtypography
\newcommand{\shortpaper}[2]{#2}

\usepackage[numbers]{natbib}

\usepackage{times,amsmath,amsfonts,amssymb,url,color,graphicx}
\usepackage{accents}
\usepackage{wrapfig}

\usepackage{algorithm}
\usepackage{algorithmic}
\algsetup{indent=2em}

\usepackage{etaremune}

\usepackage{tikz}
\usepackage{tkz-graph}  
    \usetikzlibrary{shapes.geometric}%   

\newcommand{\reals}{\mathbb{R}}

\newlength{\dhatheight}

\newtheorem{theorem}{Theorem}
\newtheorem{lemma}{Lemma}

\newcommand{\BlackBox}{\rule{1.5ex}{1.5ex}}  % end of proof
\newenvironment{proof}{\par\noindent{\bf Proof\ }}{\hfill\BlackBox\\[2mm]}
\newenvironment{subsecproof}[1]{\subsection{#1}}{\hfill\BlackBox\\[2mm]}

\DeclareMathOperator*{\E}{\mathbb{E}}
\DeclareMathOperator*{\prob}{\mathbb{P}}

\DeclareMathOperator*{\argmax}{argmax}

\newcommand{\figref}[1]{Figure~\ref{#1}}

\newcommand{\secref}[1]{Section~\ref{#1}}

\newcommand{\thmref}[1]{Theorem~\ref{#1}}

\newcommand{\lemref}[1]{Lemma~\ref{#1}}

\newcommand{\alert}[1]{\emph{#1}}

\newcommand{\var}{\mathrm{Var}}

\title{Safe, Multi-Agent, Reinforcement Learning for Autonomous Driving}
\author{Shai Shalev-Shwartz \And Shaked Shammah \And Amnon Shashua}

\begin{document}

\maketitle

\begin{abstract}
Autonomous driving is a multi-agent setting where the host vehicle must apply sophisticated negotiation skills with other road users when overtaking, giving way, merging, taking left and right turns and while pushing ahead in unstructured urban roadways.  Since there are many possible scenarios, manually tackling all possible cases will likely yield a too simplistic policy. Moreover, one must balance between unexpected behavior of other drivers/pedestrians and at the same time not to be too defensive so that normal traffic flow is maintained. 

In this paper we apply deep reinforcement learning to the problem of forming long term driving strategies. We note that there are two major challenges that make autonomous driving different from other robotic tasks. First, is the necessity for ensuring functional safety --- something that machine learning has difficulty with given that performance is optimized at the level of an expectation over many instances. Second, the Markov Decision Process model often used in robotics is problematic in our case because of unpredictable behavior of other agents in this multi-agent scenario. We make three contributions in our work. First, we show how policy gradient iterations can be used, and the variance of the gradient estimation using stochastic gradient ascent can be minimized, without Markovian assumptions. Second, we decompose the problem into a composition of a Policy for \texttt{Desires} (which is to be learned) and trajectory planning with hard constraints (which is not learned). The goal of \texttt{Desires} is to enable comfort of driving, while hard constraints guarantees the safety of driving. Third, we introduce a hierarchical temporal abstraction we call an ``Option Graph'' with a gating mechanism that significantly reduces the effective horizon and thereby reducing the variance of the gradient estimation even further. The Option Graph plays a similar role to ``structured prediction'' in supervised learning, thereby reducing sample complexity, while also playing a similar role to LSTM gating mechanisms used in supervised deep networks.

\end{abstract}

\section{Introduction}
Endowing a robotic car with the ability to form long term driving strategies, referred to as ``Driving Policy'', is key for enabling fully autonomous driving. The process of sensing , i.e., the process of forming an environmental model consisting of location of all moving and stationary objects, the position and type of path delimiters (such as curbs, barriers, and so forth), all drivable paths with their semantic meaning and all traffic signs and traffic lights around the car --- is well defined. While sensing is well understood, the definition of Driving Policy, its underlying assumptions, and its functional breakdown is less understood.  The extent of the challenge to form driving strategies that mimic human drivers is underscored by the flurry of media reports on the simplistic driving policies exhibited by current autonomous test vehicles by various practitioners (e.g. \cite{media}). In order to support autonomous capabilities a robotic driven vehicle should adopt human driving negotiation skills when overtaking, giving way, merging, taking left and right turns and while pushing ahead in unstructured urban roadways.  Since there are many possible scenarios, manually tackling all possible cases will likely yield a too simplistic policy. Moreover, one must balance between unexpected behavior of other drivers/pedestrians and at the same time not to be too defensive so that normal traffic flow is maintained. 

These challenges naturally suggest using machine learning approaches. Traditionally, machine learning approaches for planning strategies are studied under the framework of Reinforcement Learning (RL) --- see \cite{bertsekas1995dynamic,kaelbling1996reinforcement,sutton1998reinforcement,szepesvari2010algorithms} for a general overview and \cite{kober2013reinforcement} for a comprehensive review of reinforcement learning in robotics. Using machine learning, and specifically RL, raises two concerns which we address in this paper. The first is about ensuring functional safety of the Driving Policy --- something that machine learning has difficulty with given that performance is optimized at the level of an expectation over many instances. Namely, given the very low probability of an accident the only way to guarantee safety is by scaling up the variance of the parameters to be estimated and the sample complexity of the learning problem --- to a degree which becomes unwieldy to solve.
Second, the Markov Decision Process model often used in robotics is problematic in our case because of unpredictable behavior of other agents in this multi-agent scenario. 

Before explaining our approach for tackling these difficulties, we briefly describe the key idea behind most common reinforcement learning algorithms.
Typically, RL is performed in a sequence of consecutive rounds. At round $t$, the agent (a.k.a planner) observes a state, $s_t \in S$, which represents the sensing state of the system, i.e., the environmental model as mentioned above. It then should decide on an action $a_t \in A$.  After performing the action, the agent receives an immediate reward, $r_t \in \reals$, and is moved to a new state, $s_{t+1}$. The goal of the planner is to maximize the cumulative reward (maybe up to a time horizon or a discounted sum of future rewards). To do so, the planner relies on a policy, $\pi : S \to A$, which maps a state into an action.

%Most of the RL algorithms rely in some way or another on the mathematically elegant model of a Markov Decision Process (MDP), pioneered by the work of Bellman \cite{bellman1956dynamic,bellman1971introduction}. The Markovian assumption is that the distribution of $s_{t+1}$ is fully determined given $s_t$ and $a_t$. In a sense, the key advantage of the MDP model is that it allows us to couple all the future into the present using the state-action value function $Q$. That is, given that we are now in state $s$, the value of $Q^\pi(s,a)$ tells us the effect of performing action $a$ at the moment on the entire future. Therefore, the $Q$ function gives us a local measure of the quality of an action $a$, thus making the RL problem more similar to supervised learning.
%

Most of the RL algorithms rely in some way or another on the mathematically elegant model of a Markov Decision Process (MDP), pioneered by the work of Bellman \cite{bellman1956dynamic,bellman1971introduction}. The Markovian assumption is that the distribution of $s_{t+1}$ is fully determined given $s_t$ and $a_t$. This yields a closed form expression for the cumulative reward of a given policy in terms of the stationary distribution over states of the MDP.
The stationary distribution of a policy can be expressed as a solution to a linear programming problem. This yields two families of algorithms: optimizing with respect to the primal problem, which is called policy search, and optimizing with respect to the dual problem, whose variables are called the \emph{value function}, $V^\pi$. The value function determines the expected cumulative reward if we start the MDP from the initial state $s$, and from there on pick actions according to $\pi$. A related quantity is the state-action value function, $Q^\pi(s,a)$, which determines the cumulative reward if we start from state $s$, immediately pick action $a$, and from there on pick actions according to $\pi$. The $Q$ function gives rise to a crisp characterization of the optimal policy (using the so called Bellman's equation), and in particular it shows that the optimal policy is a deterministic function from $S$ to $A$ (in fact, it is the greedy policy with respect to the optimal $Q$ function). 

In a sense, the key advantage of the MDP model is that it allows us to couple all the future into the present using the $Q$ function. That is, given that we are now in state $s$, the value of $Q^\pi(s,a)$ tells us the effect of performing action $a$ at the moment on the entire future. Therefore, the $Q$ function gives us a local measure of the quality of an action $a$, thus making the RL problem more similar to supervised learning. 

Most reinforcement learning algorithms approximate the $V$ function or the $Q$ function in one way or another. Value iteration algorithms, e.g. the $Q$ learning algorithm \cite{watkins1992q}, relies on the fact that the $V$ and $Q$ functions of the optimal policy are fixed points of some operators derived from Bellman's equation. Actor-critic policy iteration algorithms aim to learn a policy in an iterative way, where at iteration $t$, the ``critic'' estimates $Q^{\pi_t}$ and based on this, the ``actor'' improves the policy. 

Despite the mathematical elegancy of MDPs and the conveniency of switching to the $Q$ function representation, there are several limitations of this approach. First, as noted in \cite{kober2013reinforcement}, usually in robotics, we may only be able to find some approximate notion of a Markovian behaving state. Furthermore, the transition of states depends not only on the agent's action, but also on actions of other players in the environment. For example, in the context of autonomous driving,  while the dynamic of the autonomous vehicle is clearly Markovian, the next state depends on the behavior of the other road users (vehicles, pedestrians, cyclists), which is not necessarily Markovian. One possible solution to this problem is to use partially observed MDPs~\cite{white1991survey}, in which we still assume that there is a Markovian state, but we only get to see an observation that is distributed according to the hidden state. A more direct approach considers game theoretical generalizations of MDPs, for example the Stochastic Games framework. Indeed, some of the algorithms for MDPs were generalized to multi-agents games. For example, the minimax-Q learning \cite{littman1994markov} or the Nash-Q learning \cite{hu2003nash}. Other approaches to Stochastic Games are explicit modeling of the other players, that goes back to Brown's fictitious play~\cite{brown1951iterative}, and vanishing regret learning algorithms \cite{hart2000simple,CesaBianchiLu06}.  See also \cite{uther1997adversarial, Thrun95a,kearns2002near,brafman2003r}. As noted in \cite{shoham2007if}, learning in multi-agent setting is inherently more complex than in the single agent setting. Taken together, in the context of autonomous driving, given the unpredictable behavior of other road users, the MDP framework and its extensions are problematic in the least and could yield impractical RL algorithms. 

When it comes to categories of RL algorithms and how they handle the Markov assumption, we can divide them into four groups: 
%\begin{itemize}
\begin{list}{\labelitemi}{\leftmargin=0.5em\itemsep=0.2pt\parsep=0pt}
\item Algorithms that estimate the Value or Q function -- those clearly are defined solely in the context of MDP.
\item Policy based learning methods where, for example, the gradient of the policy $\pi$ is 
estimated using the likelihood ratio trick (cf. \cite{Aleksandrov-etal68,Sutton-etal99,Peters-Schaal08}) and thereby the learning of $\pi$ is an iterative process where at each iteration the agent interacts with the environment while acting based on the current Policy estimation. Policy gradient methods are derived using the Markov assumption, but we will see later that this is not necessarily required. 
\item Algorithms that learn the dynamics of the process, namely, the function that takes $(s_t,a_t)$ and yields a distribution over the next state $s_{t+1}$. Those are known as Model-based methods, and those clearly rely on the Markov assumption.
\item Behavior cloning (Imitation) methods. The Imitation approach simply requires a training set of examples 
of the form $(s_t,a_t)$, where  $a_t$ is the action of the human driver (cf. \cite{Nvidia}). One can then use supervised learning to learn a policy $\pi$ such that $\pi(s_t) \approx a_t$. Clearly there is no Markov assumption involved in the process. The problem with Imitation is that different human drivers, and even the same human, are not deterministic in their policy choices. Hence, learning a function for which $\|\pi(s_t)-a_t\|$ is very small is often infeasible. And, once we have small errors, they might accumulate over time and yield large errors. 
\end{list}
%\end{itemize}
Our first observation (detailed in sec.~\ref{sec:markov}) is that Policy Gradient does not really require the Markov assumption and furthermore that some methods for reducing the variance of the gradient estimator (cf. \cite{schulman2015high}) would not require Markov assumptions as well. Taken together, the RL algorithm could be initialized through Imitation and then updated using an iterative Policy Gradient approach without the Markov assumption. 

The second contribution of the paper is a method for guaranteeing functional safety of the Driving Policy outcome. Given the very small probability $p \ll 1$ of an accident, the corresponding reward of a trajectory leading to an accident should be much smaller than $-1/p$, thus generating a very high variance of the gradient estimator (see Lemma~\ref{lemma:safe}). Regardless of the means of reducing variance, as detailed in sec.~\ref{sec:markov}, the variance of the gradient not only depends on the behavior of the reward but also on the horizon (time steps) required for making decisions. Our proposal for functional safety is twofold. First we decompose the Policy function into a composition of a Policy for \texttt{Desires} (which is to be learned) and trajectory planning with hard constraints (which is not learned). The goal of \texttt{Desires} is to enable comfort of driving, while hard constraints guarantees the safety of driving (detailed in sec.~\ref{sec:safe}). Second, following the options mechanism of \cite{sutton1999between} we employ a hierarchical temporal abstraction we call an ``Option Graph'' with a gating mechanism that significantly reduces the effective horizon and thereby reducing the variance of the gradient estimation even further (detailed in sec.~\ref{sec:option}). The Option Graph plays a similar role to ``structured prediction'' in supervised learning (e.g. \cite{taskar2005learning}), thereby reducing sample complexity, while also playing a similar role to LSTM \cite{hochreiter1997long} gating mechanisms used in supervised deep networks. The use of options for skill reuse was also been recently studied in \cite{tessler2016deep}, where hierarchical deep Q networks for skill reuse have been proposed.  Finally, in sec.~\ref{sec:exp}, we demonstrate the application of our algorithm on a double merging maneuver which is notoriously difficult to execute using conventional motion and path planning approaches.

Safe reinforcement learning was also studied recently in \cite{ammar2015safe}. Their approach involves first optimizing the expected reward (using policy gradient) and then applying a (Bregman) projection of the solution onto a set of linear constraints. This approach is different from our approach. In particular, it assumes that the hard constraints on safety can be expressed as linear constraints on the parameter vector $\theta$.  In our case $\theta$ are the weights of a deep network and the hard constraints involve highly non-linear dependency on $\theta$. Therefore, convex-based approaches are not applicable to our problem.

\section{Reinforcement Learning without Markovian Assumption}
\label{sec:markov}

We begin with setting up the RL notations geared towards deriving a Policy Gradient method with variance reduction while not making any Markov assumptions. We follow the REINFORCE \cite{williams1992simple} likelihood ratio trick and make a very modest contribution --- more an observation than a novel derivation --- that Markov assumptions on the environment are not required. Let $\cal S$ be our state space which contains the ``environmental model'' around the vehicle generated from interpreting sensory information and any additional useful information such as the kinematics of moving objects from previous frames. We use the term ``state space'' in order not to introduce new terminology but we actually mean a state vector in an agnostic sense, without the Markov assumptions --- simply a collection of information around the vehicle generated at a particular time stamp. Let $\cal A$ denote the action space, where at this point we will keep it abstract and later in sec.~\ref{sec:option} we will introduce a specific discrete action space for selecting ``desires'' tailored to the domain of autonomous driving. The hypothesis class of parametric stochastic policies is denoted by $\{\pi_\theta : \theta
 \in \Theta\}$, where for all $\theta,s$ we have $\sum_a\pi_\theta(a|s) = 1$, and we assume that $\pi_\theta(a|s)$ is differentiable w.r.t. $\theta$. Note that we have chosen a class of policies $\pi_\theta:{\cal S} \times {\cal A} \to [0,1]$ as part of an architectural design choice, i.e., that the (distribution over) action $a_t$ at time $t$ is determined by the agnostic state $s_t$ and in particular, given the differentiability over $\theta$, the policy $\pi_\theta$ is implemented by a deep layered network. In other words, we are not claiming that the optimal policy is necessarily contained in the hypothesis class but that ``good enough'' policies can be modeled using a deep network whose input layer consists of $s_t$. The theory below does not depend on the nature of the hypothesis class and any other design choices can be substituted --- for example, $\pi_\theta(a_t| s_t, s_{t-1})$ would correspond to the class of recurrent neural networks (RNN).

Let $\bar{s} = ((s_1,a_1),\ldots,(s_T,a_T))$ define a sequence (trajectory) of state-action over a time period sufficient for long-term planning, and let $\bar{s}_{i:j} = ((s_i,a_i),\ldots,(s_j,a_j))$ denote a sub-trajectory from time stamp $i$ to time stamp $j$. Let $P_\theta(\bar{s})$ be the probability of trajectory $\bar{s}$ when actions
  are chosen according to the policy $\pi_\theta$ and there are \alert{no other assumptions on the environment}. The total reward associated with the trajectory $\bar s$ is denoted by $R(\bar{s})$ which can be any function of $\bar{s}$. For example, $R$ can be a function of the immediate rewards, $r_1,\ldots,r_T$, such as $R(\bar{s})=\sum_t r_t$ or the discounted reward $R(\bar{s})=\sum_t \gamma^t r_t$ for $\gamma\in (0,1]$. But, any reward function of $\bar s$ can be used and therefore we can keep it abstract. 
Finally, the learning problem is:
\[
\argmax_{\theta \in \Theta}  \E_{\bar{s} \sim P_\theta}
[R(\bar{s})] 
\]

The gradient policy theorem below follows the standard likelihood ratio trick (e.g., \cite{Aleksandrov-etal68,Glynn87}) and the formula is well known, but in the proof (which follows the proof in \cite{Peters-Schaal08}), we make the observation that Markov assumptions on the environment are not required for the validity of the policy gradient estimator:
\begin{theorem} \label{thm:pg}
Denote 
\begin{equation}
\hat{\nabla}(\bar{s}) = R(\bar{s})
\sum_{t=1}^T \nabla_\theta \log(\pi_\theta(a_t|s_t))
\label{eq:pg}
\end{equation}
Then, 
$\E_{\bar{s} \sim P_\theta} \hat{\nabla}(\bar{s}) = \nabla \E_{\bar{s} \sim P_\theta}
[R(\bar{s})] $.
\end{theorem}

The gradient policy theorem shows that it is possible to obtain an unbiased estimate of the gradient of the expected total reward \cite{williams1992simple,Sutton-etal99, greensmith2004variance}, thereby using noisy gradient estimates in a stochastic gradient ascent/descent (SGD) algorithm for training a deep network representing the policy $\pi_\theta$. Unfortunately, the variance of the gradient estimator scales unfavorably with the time horizon $T$ and moreover due to the very low probability of critical ``corner'' cases, such as the probability of an accident $p$, the immediate reward $-r$ must satisfy $r\gg 1/p$ and in turn the variance of the random variable $R(\bar{s})$ grows with $pr^2$, i.e., much larger than $1/p$ (see sec.~\ref{sec:safe} and Lemma~\ref{lemma:safe}).
High variance of the gradient has a detrimental effect on the convergence rate of SGD \cite{moulines2011non,
  shalev2013stochastic, johnson2013accelerating, shalev2016sdca,
  needell2014stochastic} and given the nature of our problem domain, with extremely low-probability corner cases, the effect of an extremely high variance could bring about bad policy solutions.

We approach the variance problem along three thrusts. First, we use base-line subtraction methods (which goes back to \cite{williams1992simple}) for variance reduction. Second, we deal with the variance due to ``corner'' cases by decomposing the policy into a learnable part and a non-learnable part, the latter induces hard constraints on functional safety.  Last, we introduce a temporal abstraction method with a gating mechanism we call an ``option graph'' to ameliorate the effect of the time horizon $T$ on the variance. In \secref{sec:variance} we focus on base-line subtraction, derive the optimal baseline (following \cite{Peters-Schaal08}) and generalize the recent results of \cite{schulman2015high} to a non-Markovian setting. In the next section we deal with variance due to ``corner cases''.

\shortpaper{}{\section{Variance Reduction} \label{sec:variance} Consider again policy gradient estimate $\hat{\nabla}(\bar{s})$ introduced in eqn.\ref{eq:pg}. The baseline subtraction method reduces the variance of $R(\bar{s})$ by subtracting a scalar $b_{t,i}$ from $R(\bar{s})$: 
\[
\hat{\nabla}_i(\bar{s}) =  \sum_{t=1}^T (R(\bar{s})-b_{t,i})\nabla_{\theta_i} \log(\pi(a_t|s_t)) ~.
\]
The Lemma below describes the conditions on $b_{t,i}$ for the baseline subtraction to work:

\begin{lemma} \label{lem:Bt}
For every $t$ and $i$, let $b_{t,i}$ be a scalar that does not depend on $a_t,\bar{s}_{t+1:T}$, but may depend on 
$\bar{s}_{1:t-1}, s_t,$ and on $\theta$. Then, 
\[
\E_{\bar{s} \sim P_\theta} \left[ \sum_t b_{t,i} \nabla_{\theta_i} \log
  \pi_\theta(a_t|s_t) \right] = 0 ~.
\]
\end{lemma}

The optimal baseline, one that would reduce the variance the most, can be derived following \cite{Peters-Schaal08}:
\[
\E \hat{\nabla}_i(\bar{s})^2 = \E \left( \sum_{t=1}^T (R(\bar{s})-b_{t,i}) \nabla_{\theta_i} \log(\pi(a_t|s_t)) \right)^2
\]
Taking the derivative w.r.t. $b_{\tau,i}$ and comparing to zero we obtain the following equation for the optimal baseline:
\[
\E \left( \sum_{t=1}^T (R(\bar{s})-b_{t,i}) \nabla_{\theta_i} \log(\pi(a_t|s_t)) \right) \left( \nabla_{\theta_i} \log(\pi(a_\tau|s_\tau)) \right) = 0 ~.
\]
This can be written as $X b = y$, where $X$ is a $T \times T$ matrix with $X_{\tau,t} = \E\, \nabla_{\theta_i} \log(\pi(a_t|s_t)) \nabla_{\theta_i} \log(\pi(a_\tau|s_\tau)) $ and $y$ is a $T$ dimensional vector with $y_t = \E R(\bar{s}) \nabla_{\theta_i} \log(\pi(a_t|s_t)) \nabla_{\theta_i} \log(\pi(a_\tau|s_\tau))$. We can estimate $X$ and $y$ from a mini-batch of episodes and then set $b_{\cdot,i}$ to be $X^{-1} y$. 
A more efficient approach is to think about the problem of finding the baseline as an online linear regression problem and have a separate process that update $b_{\cdot,i}$ in an online manner~\cite{shalev2011online}.

Many policy gradient variants \cite{schulman2015high} replace $R(\bar{s})$ with the Q-function, which assumes the Markovian setting.  The following lemma gives a non-Markovian analogue of the $Q$ function.
\begin{lemma} \label{lem:Qestimator} Define
\begin{equation} \label{eqn:Qdef}
Q_\theta(\bar{s}_{1:t}) := \sum_{\bar{s}_{t+1:T}} P_\theta( \bar{s}_{t+1:T}|\bar{s}_{1:t}) \,R(\bar{s}) ~.
\end{equation} 
 Let $\xi$ be a random variable and let 
$\hat{Q}_\theta(\bar{s}_{1:t},\xi)$ be a function such that $\E_{\xi} \hat{Q}_\theta(\bar{s}_{1:t},\xi) =  Q_\theta(\bar{s}_{1:t})$ (in particular, we can take $\xi$ to be empty and then $\hat{Q} \equiv Q$). 
Then,
\[
\E_{\bar{s} \sim P_\theta, \xi}\left[ \sum_{t=1}^T \hat{Q}_\theta(\bar{s}_{1:t},\xi) \nabla_\theta \log(\pi_\theta(a_t|s_t)) \right] = \nabla \E_{\bar{s} \sim P_\theta}
[R(\bar{s})] ~.
\]
\end{lemma}

Observe that the following analogue of the value function for the non-Markovian setting,
\begin{equation} \label{eqn:nmVdef}
V_\theta(\bar{s}_{1:t-1},s_t) = \sum_{a_t} \pi_\theta(a_t|s_t) \, Q_\theta(\bar{s}_{1:t}) ~,
\end{equation}
satisfies the conditions of Lemma~\ref{lem:Bt}. Therefore, we can also replace $R(\bar{s})$ with an analogue of the so-called Advantage function, 
$A(\bar{s}_{1:t}) = Q_\theta(\bar{s}_{1:t})  - V_\theta(\bar{s}_{1:t-1},s_t)$. 
The advantage function, and generalization of it, are often used in actor-critique policy gradient implementations (see for example \cite{schulman2015high}). In the 
non-Markovian setting considered in this paper, the Advantage function is more complicated to estimate, and therefore, in our experiments, we use estimators that involve the term $R(\bar{s})-b_{t,i}$, where $b_{t,i}$ is estimated using online linear regression.  

}

\section{Safe Reinforcement Learning}
\label{sec:safe}

In the previous section we have shown how to optimize the reinforcement learning objective by policy stochastic gradient ascent. Recall that we have defined the objective to be $\E_{\bar{s} \sim P_\theta} R(\bar{s})$, that is, the \emph{expected} reward. Objectives that involve expectation are common in machine learning. We now argue that this objective poses a functional safety problem.

Consider a reward function for which $R(\bar{s})=-r$ for trajectories that represent a rare ``corner'' event which we would like to avoid, such as an accident, and $R(\bar{s})\in[-1,1]$ for the rest of the trajectories. For concreteness, suppose that our goal is to learn to perform an overtake maneuver. Normally, in an accident free trajectory, $R(\bar{s})$ would reward successful, smooth, takeovers and penalize against staying in lane without completing the takeover --- hence the range $[-1,1]$. If a sequence, $\bar{s}$, represents an accident, we would like the reward $-r$ to provide a sufficiently high penalty to discourage such occurrences. The question is what should be the value of $r$ to ensure accident-free driving?

Observe that  the effect of an accident on $\E[R(\bar{s})]$ is the additive term $-pr$, where $p$ is the probability mass of trajectories with an accident event. If this term is negligible, i.e., $p \ll 1/r$, then the learner might prefer a policy that performs an accident (or adopt in general a reckless driving policy) in order to fulfill the takeover maneuver successfully more often than a policy that would be more defensive at the expense of having some takeover maneuvers not completed successfully.  
In other words, if we want to make sure that the probability of accidents is at most $p$ then we must set $r \gg 1/p$. Since we would like $p$ to be extremely small (say, $p=10^{-9}$), we obtain that $r$ must be extremely large. Recall that in policy gradient we estimate the gradient of $\E[R(\bar{s})]$. The following lemma shows that the variance of the random variable $R(\bar{s})$ grows with $pr^2$, which is larger than $r$ for $r \gg 1/p$. Hence, even estimating the objective is difficult, let alone its gradient.
\begin{lemma}
\label{lemma:safe}
Let $\pi_\theta$ be a policy and let $p,r$ be scalars such that with probability $p$ we have $R(\bar{s}) = -r$ and with probability $1-p$ we have $R(\bar{s}) \in [-1,1]$. Then,
\[
\var[R(\bar{s})] \ge pr^2 - (pr+(1-p))^2
= (p-p^2)r^2 -2p(1-p)r - (1-p)^2 \approx pr^2 ~,
\]
where the last approximation holds for the case $r \ge 1/p$.
\end{lemma}

The above discussion shows that an objective of the form $\E[R(\bar{s})]$ cannot ensure functional safety without causing a serious variance problem. The baseline subtraction method for variance reduction would not offer a sufficient remedy to the problem because we would be shifting the problem from very high variance of $R(\bar{s})$ to equally high variance of the baseline constants whose estimation would equally suffer numerical instabilities. Moreover, if the probability of an accident is $p$ then on average we should sample at least $1/p$ sequences before obtaining an accident event. This immediately implies a lower bound of $1/p$ samples of sequences for any learning algorithm that aims at minimizing $\E[R(\bar{s})]$. We therefore face a fundamental problem whose solution must be found in a new architectural design and formalism of the system rather than through numerical conditioning tricks. 

Our approach is based on the notion that hard constraints should be injected \emph{outside} of the learning framework. In other words, we decompose  the policy function into a learnable part and a non-learnable part. Formally, we structure the policy function as $\pi_\theta = \pi^{(T)} \circ \pi_\theta^{(D)}$, where $\pi_\theta^{(D)}$ maps the (agnostic) state space into a set of \texttt{Desires}, while $\pi^{(T)}$ maps the \texttt{Desires} into a trajectory (which determines how the car should move in a short range). The function $\pi_\theta^{(D)}$ is responsible for the comfort of driving and for making strategical decisions such as which other cars should be over-taken or given way and what is the desired position of the host car within its lane and so forth. The mapping from state to \texttt{Desires} is a policy $\pi_\theta^{(D)}$ that is being learned from experience by maximizing an expected reward. The desires produced by $\pi_\theta^{(D)}$  are translated into a cost function over driving trajectories. The function $\pi^{(T)}$, which is not being learned, is implemented by finding a trajectory that minimizes the aforementioned cost subject to hard constraints on functional safety. This decomposition allows us to always ensure functional safety while at the same time enjoying comfort driving most of the time. 

To illustrate the idea, let us consider a challenging driving scenario, which we call the \emph{double merge} scenario (see \figref{fig:doublemerge} for an illustration). In a double merge, vehicles approach the merge area from both left and right sides and, from each side, a vehicle can decide whether to merge into the other side or not. Successfully executing a  double merge in busy traffic requires significant negotiation skills and experience and is notoriously difficult to execute in a heuristic or brute force approach by enumerating all possible trajectories that could be taken by all agents in the scene.

\begin{wrapfigure}{r}{0.3\textwidth}
\begin{center}
\iffalse
\begin{tikzpicture}[scale=0.25]
\clip (-4,-10) rectangle (4,10);
\foreach \x/\s in {0.4/dashed,0.8/solid}
  \draw[thick,\s] (0-\x,5) arc (0:180:5-\x) --  (-10+\x,-5) arc (180:360:5-\x) -- cycle;
\foreach \x/\s in {0.4/dashed,0.8/solid}
  \draw[thick,\s] (0+\x,-5) arc (180:360:5-\x) --  (10-\x,5) arc (0:180:5-\x) -- cycle;

\draw[thick,dashed] (0,-5) -- (0,5);
\draw[thick,solid] (0,5) arc (0:180:5) --  (-10,-5) arc (180:360:5);
\draw[thick,solid] (0,-5) arc (180:360:5) --  (10,5) arc (0:180:5);

\end{tikzpicture} ~\hspace{2cm}~
\fi
\includegraphics[height=0.15\textheight]{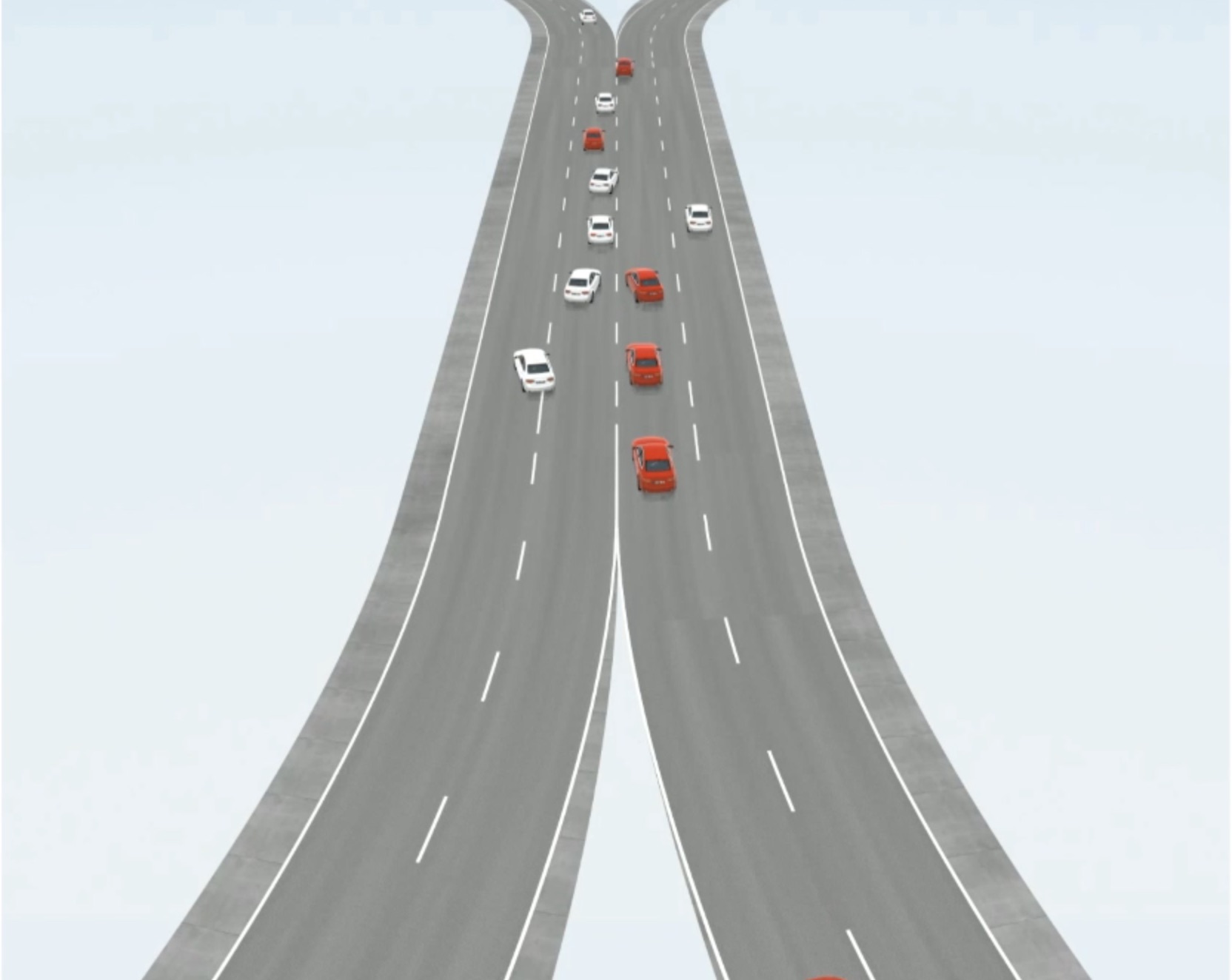}
\end{center}
\caption{\small The double merge scenario. Vehicles arrive from the left or right side to the merge area. Some vehicles should continue on their road while other vehicles should merge to the other side. In dense traffic, vehicles must negotiate the right of way.} \label{fig:doublemerge}
\end{wrapfigure}
We begin by defining the set of \texttt{Desires} $\cal D$ appropriate for the double merge maneuver. Let $\cal D$ be the Cartesian product of the following sets:
$${\cal D} = [0,v_{\max}] \times L \times \{g,t,o\}^n$$
where $[0,v_{\max}]$ is the desired target speed of the host vehicle, $L=\{1,1.5,2,2.5,3,3.5,4\}$ is the desired lateral position in lanes units where whole numbers designate lane center and fraction numbers designate lane boundaries, and $\{g,t,o\}$ are classification labels assigned to each of the $n$ other vehicles. Each of the other vehicles is assigned `g' if the host vehicle is to ``give way'' to it, or `t' to ``take way'' and `o' to maintain an offset distance to it. 

Next we describe how to translate a set of \texttt{Desires}, $(v,l,c_1,\ldots,c_n) \in {\cal D}$, into a cost function over driving trajectories. A driving trajectory is represented by $(x_1,y_1),\ldots,(x_k,y_k)$, where $(x_i,y_i)$ is the (lateral,longitudinal) location of the car (in egocentric units) at time $\tau\cdot i$. In our experiments, we set $\tau = 0.1\textrm{sec}$ and $k = 10$. The cost assigned to a trajectory will be a weighted sum of individual costs assigned to the desired speed, lateral position, and the label assigned to each of the other $n$ vehicles. Each of the individual costs are descried below.

Given a desired speed $v\in [0,v_{\max}]$, the cost of a trajectory associated with speed is $\sum_{i=2}^k (v - \|(x_i,y_i)-(x_{i-1},y_{i-1})\|/\tau\,)^2$. Given desired lateral position $l\in L$, the cost associated with desired lateral position is $\sum_{i=1}^k \textrm{dist}(x_i,y_i,l)$, where $\textrm{dist}(x,y,l)$ is the distance from the point $(x,y)$ to the lane position $l$. As to the cost due to other vehicles, for any other vehicle let $(x'_1,y'_1),\ldots,(x'_k,y'_k)$ be its predicted trajectory in the host vehicle egocentric units, and let $i$ be the earliest point for which there exists $j$ such that the distance between $(x_i,y_i)$ and $(x'_j,y'_j)$ is small (if there is no such point we let $i = \infty$). If the car is classified as ``give-way'' we would like that $\tau\,i > \tau\,j+0.5$, meaning that we will arrive to the trajectory intersection point at least $0.5$ seconds after the other vehicle will arrive to that point. A possible formula for translating the above constraint into a cost is $[\tau\,(j-i) + 0.5]_+$. Likewise, if the car is classified as ``take-way'' we would like that $\tau\,j > \tau\,i+0.5$, which is translated to the cost $[\tau\,(i-j) + 0.5]_+$. Finally, if the car is classified as ``offset'' we would like that $i$ will be $\infty$ (meaning, the trajectories will not intersect). This can be translated to a cost by penalizing on the distance between the trajectories. 

By assigning a weight to each of the aforementioned costs we obtain a single objective function for the trajectory planner, $\pi^{(T)}$. Naturally, we can also add to the objective a cost that encourages smooth driving. More importantly, we add hard constraints that ensure functional safety of the trajectory. For example, we do not allow $(x_i,y_i)$ to be off the roadway and we do not allow $(x_i,y_i)$ to be close to $(x'_j,y'_j)$ for any trajectory point $(x'_j,y'_j)$ of any other vehicle if $|i-j|$ is small. 

To summarize, we decompose the policy $\pi_\theta$ into a mapping from the agnostic state to a set of \texttt{Desires} and a mapping from the \texttt{Desires} to an actual trajectory. 
The latter mapping is not being learned and is implemented by solving an optimization problem whose cost depends on the \texttt{Desires} and whose hard constraints guarantees functional safety of the policy. It is left to explain how we learn the mapping from the agnostic state to the \texttt{Desires}, which is the topic of the next section. 

\section{Temporal Abstraction}
\label{sec:option}

In the previous section we injected prior knowledge in order to break down the problem in such a way to ensure functional safety. We saw that through RL alone a system complying with functional safety will suffer a very high and unwieldy variance on the reward $R(\bar{s})$ and this can be fixed by splitting the problem formulation into a mapping from (agnostic) state space to \texttt{Desires} using policy gradient iterations followed by a mapping to the actual trajectory which does not involve learning.
It is necessary, however, to inject even more prior knowledge into the problem and decompose the decision making into semantically meaningful components --- and this for two reasons. First, the size of ${\cal D}$ might be quite large and even continuous (in the double-merge scenario described in the previous section we had ${\cal D} = [0,v_{\max}] \times L \times \{g,t,o\}^n$).  Second, the gradient estimator involves the term $\sum_{t=1}^T \nabla_\theta \pi_\theta(a_t|s_t)$. As mentioned above, the variance grows with the time horizon $T$ \cite{Peters-Schaal08} . In our case, the value of $T$ is roughly\footnote{Suppose we work at 10Hz, the merge area is 100
  meters, we start the preparation for the merge 300 meters before it, and we drive at 16
  meters per second ($\approx$ 60 Km per hour). In this case, the value of $T$ for
  an episode is roughly $250$.} 250 which is high enough to create significant variance.

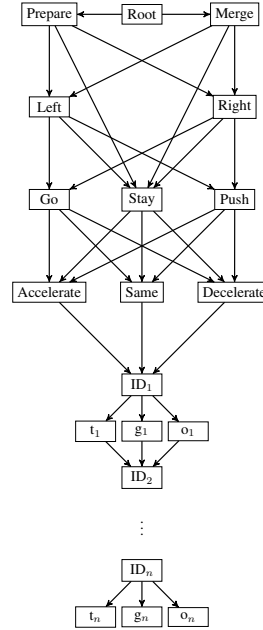
\begin{wrapfigure}{r}{0.3\textwidth}
\begin{center}
\tikzstyle{VertexStyle} = [shape  = rectangle, minimum width    = 6ex, draw]
\tikzstyle{EdgeStyle}   = [->,>=stealth']      
\resizebox{.25\textwidth}{!}{%
\begin{tikzpicture}[scale=1.1] 
\SetGraphUnit{2} 
\Vertex{Root}  
\EA(Root){Merge}
\WE(Root){Prepare}
\SO(Merge){Right}
\SO(Prepare){Left}
\SO(Left){Go}
\SO(Right){Push}
\EA(Go){Stay}
\SO(Stay){Same}
\WE(Same){Accelerate}
\EA(Same){Decelerate}
\SO[L=ID$_1$](Same){ID1}
\SetGraphUnit{1}
\SO[L=g$_1$](ID1){g1}
\WE[L=t$_1$](g1){t1}
\EA[L=o$_1$](g1){o1}
\SO[L=ID$_2$](g1){ID2}
\SO[L=$\vdots$,style={draw=white,fill=white}](ID2){vdots}
%\SetGraphUnit{2}
\SO[L=ID$_n$](vdots){IDn}
%\SetGraphUnit{1}
\SO[L=g$_n$](IDn){gn}
\WE[L=t$_n$](gn){tn}
\EA[L=o$_n$](gn){on}

\foreach \d in {Merge,Prepare} {\Edges(Root,\d)};
\foreach \s in {Prepare,Merge} \foreach \d in {Left,Stay,Right} { \Edges(\s,\d) };
\foreach \s in {Left,Right} \foreach \d in {Stay,Go,Push} {\Edges(\s,\d) };
\foreach \s in {Go,Stay,Push} \foreach \d in {Decelerate,Same,Accelerate} { \Edges(\s,\d) };
\foreach \s in {Decelerate,Same,Accelerate} {\Edges(\s,ID1)};
\foreach \i in {1,n} {\Edges(ID\i,t\i) \Edges(ID\i,g\i) \Edges(ID\i,o\i)};
\foreach \s in {g1,t1,o1} {\Edges(\s,ID2)};

\end{tikzpicture}
}
\end{center}
\caption{\small An options graph for the double merge scenario.} \label{fig:options}
\end{wrapfigure}

Our approach 
follows the \emph{options framework} due to \cite{Sutton-Precup-Singh99}.  An \emph{options graph} represents a hierarchical set of decisions organized as a Directed Acyclic Graph (DAG). There is a special node called the ``root'' of the graph. The root node is the only node that has no incoming edges. The decision process traverses the graph, starting from the root node, until it reaches a ``leaf'' node, namely, a node that the has no outgoing edges. Each internal node should implement a policy function that picks a child among its available children. There is a predefined mapping from the set of traversals over the options graph to the set of desires, ${\cal D}$. In other words, a traversal on the options graph is automatically translated into a desire in $\cal D$. Given a node $v$ in the graph, we denote by $\theta_v$ the parameter vector that specifies the policy of picking a child of $v$. Let $\theta$ be the concatenation of all the $\theta_v$, then $\pi_\theta^{(D)}$ is defined by traversing from the root of the graph to a leaf, while at each node $v$ using the policy defined by $\theta_v$ to pick a child node. 

A possible option graph for the double merge scenario is depicted in \figref{fig:options}. The root node first decides if we are within the merging area or if we are approaching it and need to prepare for it. In both cases, we need to decide whether to change lane (to left or right side) or to stay in lane. If we have decided to change lane we need to decide if we can go on and perform the lane change maneuver (the ``go'' node). If it is not possible, we can try to ``push'' our way (by aiming at being on the lane mark) or to ``stay'' in the same lane. This determines the desired lateral  position in a natural way --- for example, if we change lane from lane $2$ to lane $3$, ``go'' sets the desire lateral position to $3$, ``stay'' sets the desire lateral position to $2$, and ``push'' sets the desire lateral position to $2.5$.  Next, we should decide whether to keep the ``same'' speed, ``accelerate'' or ``decelerate''. Finally, we enter a ``chain like'' structure that goes over all the vehicles and sets their semantic meaning to a value in $\{g,t,o\}$. This sets the desires for the semantic meaning of vehicles in an obvious way. Note that we share the parameters of all the nodes in this chain (similarly to Recurrent Neural Networks). 

An immediate benefit of the options graph is the interpretability of the results. Another immediate benefit is that we rely on the decomposable structure of the set $\cal D$ and therefore the policy at each node should choose between a small number of possibilities. Finally, the structure allows us to reduce the variance of the policy gradient estimator.  We next elaborate on the last point. 

As mentioned previously, the length of an episode in the double merge scenario is roughly $T = 250$ steps. This number comes from the fact that on one hand we would like to have enough time to see the consequences of our actions (e.g., if we decided to change lane as a preparation for the merge, we will see the benefit only after a successful completion of the merge), while on the other hand due to the dynamic of driving, we must make decisions at a fast enough frequency ($10$Hz in our case). The options graph enables us to decrease the effective value of $T$ by two complementary ways. First, given higher level decisions we can define a reward for lower level decisions while taking into account much shorter episodes.  For example, when we have already picked the ``lane change'' and ``go'' nodes, we can learn the  policy for assigning semantic meaning to vehicles by looking at episodes of 2-3 seconds (meaning that $T$ becomes $20-30$ instead of $250$). 
Second, for high level decisions (such as whether to change lane or to stay in the same lane), we do not need to make decisions every $0.1$ seconds. Instead, we can either make decisions at a lower frequency (e.g., every second), or implement an ``option termination'' function, and then the gradient is calculated only after every termination of the option. In both cases, the effective value of $T$ is again an order of magnitude smaller than its original value. All in all, the estimator at every node depends on a value of $T$ which is an order of magnitude smaller than the original $250$ steps, which immediately transfers to a smaller variance \cite{Mann2015b}.

To summarize, we introduced the options graph as a way to breakdown the problem into semantically meaningful components where the  \texttt{Desires} are defined through a traversal over a DAG. At each step along the way the learner maps the state space to a small subset of \texttt{Desires} thereby effectively decreasing the time horizon to much smaller sequences and at the same time reducing the output space for the learning problem. The aggregated effect is both in reducing variance and sample complexity of the learning problem.

\section{Experimental Demonstration}
\label{sec:exp}

The purpose of this section is to give a sense of how a challenging negotiation scenario is handled by our framework. The experiment involves propriety software modules (to produce the sensing state, the simulation, and the trajectory planner) and data (for the learning-by-imitation part). It therefore should be regarded as a demonstration rather than a reproducible experiment. We leave to future work the task of conducting a reproducible experiment, with a comparison to other approaches. 

We experimented with the double-merge scenario described in \secref{sec:safe} (see again \figref{fig:doublemerge}). This is a challenging negotiation task as cars from both sides have a strong incentive to merge, and failure to merge in time leads to ending on the wrong side of the intersection. In addition, the reward $R(\bar s)$ associated with a trajectory needs to account not only for the success or failure of the merge operation but also for smoothness of the trajectory control and the comfort  level of all other vehicles in the scene. In other words, the goal of the RL learner is not only to succeed with the merge maneuver but also accomplish it in a smooth manner and without disrupting the driving patterns of other vehicles.

We relied on the following sensing information. The static part of the environment is represented as the geometry of lanes and the free space (all in ego-centric units). Each agent also observes the location, velocity, and heading of every other car which is within 100 meters away from it. Finally, $300$ meters before the merging area the agent receives the side it should be after the merge ('left' or 'right').  For the trajectory planner, $\pi^{(T)}$, we used an optimization algorithm based on dynamic programming.
We used the option graph described in \figref{fig:options}. Recall that we should define a policy function for every node of our option graph.  We initialized the policy at all nodes using imitation learning. Each policy function, associated with every node of the option graph, is represented by a neural network with three fully connected hidden layers. Note that data collected from a human driver only contains the final maneuver, but we do not observe a traversal on the option graph. For some of the nodes, we can infer the labels from the data in a relatively straight forward manner. For example, the classification of vehicles to ``give-way'', ``take-way'', and ``offset'' can be inferred from the future position of the host vehicle relative to the other vehicles. For the remaining nodes we used an implicit supervision. Namely, our option graph induces a probability over future trajectories and we train it by maximizing the (log) probability of the trajectory that was chosen by the human driver.  Fortunately, deep learning is quite good in dealing with hidden variables and the imitation process succeeded to learn a reasonable initialization point for the policy. See \cite{SuppVideos} for some videos. 
For the policy gradient updates we used a simulator (initialized using imitation learning) with self-play enhancement. Namely, we partitioned the set of agents to two sets, $A$ and $B$. Set $A$ was used as reference players while set $B$ was used for the policy gradient learning process. When the learning process converged, we used set $B$ as the reference players and used the set $A$ for the learning process. The alternating process of switching the roles of the sets continued for $10$ rounds. 
See \cite{SuppVideos} for resulting videos. 

\shortpaper{}{%
\subsection*{Acknowledgements} We thank Moritz Werling, Daniel Althoff, and Andreas Lawitz for helpful discussions. 
}

\bibliographystyle{plainnat}
\bibliography{bib}

\appendix

\shortpaper{\section{Variance Reduction} \label{sec:variance} }{}

\section{Proofs}

\begin{subsecproof}{Proof of \thmref{thm:pg}}
The policy $\pi_\theta$ induces a probability distribution over sequences as
follows: given a sequence $\bar{s} =
(s_1,a_1),\ldots,(s_T,a_T)$, we have
\[
P_\theta(\bar{s}) ~=~ \prod_{t=1}^T \prob[s_t | \bar{s}_{1:t-1}] \, \pi_\theta(a_t|s_t) ~.
\]
Note that in deriving the above expression we make no assumptions on 
$\prob[s_t | \bar{s}_{1:t-1}]$. This stands in contrast to Markov
Decision Processes, in which it is assumed that $s_t$ is independent on the past given $(s_{t-1},a_{t-1})$. 
The only assumption we make is that the (random) choice of $a_t$ is solely based on $s_t$, which comes from our architectural design choice of the hypothesis space of policy functions. The remainder of the proof employs the standard likelihood ratio trick (e.g., \cite{Aleksandrov-etal68,Glynn87}) with the observation that since $\prob[s_t | \bar{s}_{1:t-1}]$ does not depend on the parameters $\theta$ it gets eliminated in the policy gradient. This is detailed below for the sake of completeness:
\begin{align*}
  \nabla_\theta \E_{\bar{s} \sim P_\theta}
[R(\bar{s})]  &= \nabla_\theta \sum_{\bar{s}}
                             P_\theta(\bar{s}) R(\bar{s}) &&
                                                             \textrm{(definition
                                                             of expectation)}\\
&= \sum_{\bar{s}} R(\bar{s}) \nabla_\theta P_\theta(\bar{s})  &&
                                                             \textrm{(linearity
                                                                 of derivation)}\\
&= \sum_{\bar{s}} P_\theta(\bar{s}) \,R(\bar{s}) \frac{\nabla_\theta
  P_\theta(\bar{s})}{P_\theta(\bar{s})} 
&& \textrm{(multiply and divide by}~  P_\theta(\bar{s}))  \\
&= \sum_{\bar{s}} P_\theta(\bar{s}) \,R(\bar{s}) \nabla_\theta
  \log(P_\theta(\bar{s}))  
&& \textrm{(derivative of the log)} \\
&= \sum_{\bar{s}} P_\theta(\bar{s}) \,R(\bar{s}) \nabla_\theta
  \left( \sum_{t=1}^T \log(\prob[s_t|\bar{s}_{1:t-1}]) + \sum_{t=1}^T \log(\pi_\theta(s_t,a_t))\right)
&& \textrm{(def of } P_\theta) \\
&= \sum_{\bar{s}} P_\theta(\bar{s}) \,R(\bar{s}) 
  \left( \sum_{t=1}^T \nabla_\theta\log(\prob[s_t|\bar{s}_{1:t-1}]) + \sum_{t=1}^T \nabla_\theta\log(\pi_\theta(s_t,a_t))\right)
&& \textrm{(linearity of derivative)} \\
&= \sum_{\bar{s}} P_\theta(\bar{s}) \,R(\bar{s}) 
\left( 0 + \sum_{t=1}^T  \nabla_\theta  \log(\pi_\theta(s_t,a_t)) \right)
\\
&= \E_{\bar{s} \sim P_{\theta}}\left[ R(\bar{s}) \sum_{t=1}^T  \nabla_\theta
  \log(\pi_\theta(s_t,a_t)) \right] ~.
\end{align*}
This concludes our proof.
\end{subsecproof}

\begin{subsecproof}{Proof of \lemref{lem:Bt}}
\begin{align*}
&\E_{\bar{s}}\left[ b_{t,i} \nabla_{\theta_i}
  \log(\pi_\theta(s_t,a_t)) \right] = 
\sum_{\bar{s}} P_\theta(\bar{s}) b_{t,i} \nabla_{\theta_i}
  \log(\pi_\theta(s_t,a_t)) \\
&= \sum_{\bar{s}}  P_\theta( \bar{s}_{1:t-1},s_t ) \, \pi_\theta( a_t
  | s_t)  \, P_\theta(\bar{s}_{t+1:T} | \bar{s}_{1:t}) \, b_{t,i} \nabla_{\theta_i}
  \log(\pi_\theta(s_t,a_t)) \\
&= \sum_{\bar{s}_{1:t-1},s_t} P_\theta( \bar{s}_{1:t-1},s_t ) \,b_{t,i}\, \sum_{a_t} \pi_\theta( a_t
  | s_t) \,\nabla_{\theta_i}
  \log(\pi_\theta(s_t,a_t))\, \sum_{\bar{s}_{t+1:T}} P_\theta(\bar{s}_{t+1:T} | \bar{s}_{1:t}) \\
&= \sum_{\bar{s}_{1:t-1},s_t} P_\theta( \bar{s}_{1:t-1},s_t ) \,b_{t,i}\, \left[\sum_{a_t} \pi_\theta( a_t
  | s_t) \,\nabla_{\theta_i}
  \log(\pi_\theta(s_t,a_t)) \right] ~.
\end{align*}
By \lemref{lem:variance_lem}, the term in the parentheses is $0$, which concludes our proof. 
\end{subsecproof}

\begin{subsecproof}{Proof of \lemref{lem:Qestimator}}
We have:
\begin{align*}
&\E_{\bar{s} \sim P_\theta,\xi}\left[ \sum_{t=1}^T \hat{Q}_\theta(\bar{s}_{1:t},\xi) \nabla_\theta \log(\pi_\theta(a_t|s_t)) \right]  \\
&= \sum_{t=1}^T \sum_{\bar{s}_{1:t}} P_\theta(\bar{s}_{1:t}) \E_\xi \hat{Q}_\theta(\bar{s}_{1:t},\xi) \nabla_\theta \log(\pi_\theta(a_t|s_t)) \sum_{\bar{s}_{t+1:T}} P_\theta(\bar{s}_{t+1:T} | \bar{s}_{1:t}) \\
&= \sum_{t=1}^T \sum_{\bar{s}_{1:t}} P_\theta(\bar{s}_{1:t}) Q_\theta(\bar{s}_{1:t}) \nabla_\theta \log(\pi_\theta(a_t|s_t)) \\
&= \sum_{t=1}^T \sum_{\bar{s}_{1:t}} P_\theta(\bar{s}_{1:t}) \nabla_\theta \log(\pi_\theta(a_t|s_t)) \sum_{\bar{s}_{t+1:T}} P_\theta(\bar{s}_{t+1:T}|\bar{s}_{1:t}) R(\bar{s}) \\
&= \E_{\bar{s}}\left[ R(\bar{s}) \sum_{t=1}^T  \nabla_\theta \log(\pi_\theta(a_t|s_t))\right] ~.
\end{align*}
The claim follows from \thmref{thm:pg}.
\end{subsecproof}

\begin{subsecproof}{Proof of \lemref{lemma:safe}}
We have
\[
\E[R(\bar{s})] \in [-pr - (1-p), -pr + (1-p)] ~~~\Rightarrow~~~
\E[R(\bar{s})]^2 \le (pr+(1-p))^2
\]
and $
\E[R(\bar{s})^2] \ge p r^2$. 
The claim follows since $\var[R(\bar{s})] = \E[R(\bar{s})^2]  - \E[R(\bar{s})]^2$.
\end{subsecproof}

\subsection{Technical Lemmas}

\begin{lemma} \label{lem:variance_lem}
Suppose that $\pi_\theta$ is a function such that for every $\theta$
and $s$ we have $\sum_a \pi_\theta(a|s) = 1$. Then, 
\[
\sum_a \pi_\theta(a|s) \nabla_\theta \log \pi_\theta(a|s)  = 0
\]
\end{lemma}
\begin{proof}
\begin{align*}
\sum_a \pi_\theta(a|s) \nabla \log \pi_\theta(a|s) &= 
\sum_a \pi_\theta(a|s)  \frac{\nabla_\theta \pi_\theta(a|s)}{\pi_\theta(a|s)}\\
&= \sum_a \nabla_\theta \pi_\theta(a|s) = \nabla_\theta \sum_a \pi_\theta(a|s) = \nabla_\theta 1 = 0 
\end{align*}
\end{proof}

\end{document}